\documentclass{article}

\usepackage{graphicx}
\usepackage{subfigure}
\usepackage{booktabs} 

\usepackage{hyperref}

\usepackage[accepted]{icml2020}

\usepackage[utf8]{inputenc} 
\usepackage{url}            
\usepackage{amsfonts}       
\usepackage{nicefrac}       
\usepackage{subfigure}
\usepackage{amsmath}
\usepackage{amsthm}
\usepackage{amssymb}

\newtheorem{remark}{Remark}
\newtheorem{prop}{Proposition}
\usepackage{makecell}
\usepackage{dblfloatfix}    

\icmltitlerunning{Submission and Formatting Instructions for ICML 2020}

\begin{document}

\twocolumn[
\icmltitle{Learning Efficient Multi-agent Communication: \\ An Information Bottleneck Approach}



\icmlsetsymbol{equal}{*}

\begin{icmlauthorlist}
\icmlauthor{Rundong Wang}{ntu}
\icmlauthor{Xu He}{ntu}
\icmlauthor{Runsheng Yu}{ntu}
\icmlauthor{Wei qiu}{ntu}
\icmlauthor{Bo An}{ntu}
\icmlauthor{Zinovi Rabinovich}{ntu}
\end{icmlauthorlist}

\icmlaffiliation{ntu}{School of Computer Science and Engineering, Nanyang Technological University, Singapore}

\icmlcorrespondingauthor{Rundong Wang}{rundong001@e.ntu.edu.sg}

\icmlkeywords{Machine Learning, ICML}

\vskip 0.3in
]



\printAffiliationsAndNotice{\icmlEqualContribution} 

\begin{abstract}


We consider the problem of the limited-bandwidth communication for multi-agent reinforcement learning, where agents cooperate with the assistance of a communication protocol and a scheduler. The protocol and scheduler jointly determine \textit{which} agent is communicating \textit{what} message and to \textit{whom}. Under the limited bandwidth constraint, a communication protocol is required to generate informative messages. Meanwhile, an unnecessary communication connection should not be established because it occupies limited resources in vain. In this paper, we develop an {\it Informative Multi-Agent Communication} (IMAC) method to learn efficient communication protocols as well as scheduling. First, from the perspective of communication theory, we prove that the limited bandwidth constraint requires low-entropy messages throughout the transmission. Then inspired by the information bottleneck principle, we learn a valuable and compact communication protocol and a weight-based scheduler. To demonstrate the efficiency of our method, we conduct extensive experiments in various cooperative and competitive multi-agent tasks with different numbers of agents and different bandwidths. We show that IMAC converges faster and leads to efficient communication among agents under the limited bandwidth as compared to many baseline methods.

\end{abstract}


\section{Introduction}

Multi-agent reinforcement learning (MARL) has long been a go-to tool in complex robotic and strategic domains~\cite{robocup,dota2}. In these scenarios,  communicated information enables action and belief correlation that benefits a group's cooperation. Therefore, many recent works in the field of multi-agent communication focus on learning what messages \cite{foerster2016learning,sukhbaatar2016learning,peng2017multiagent} to send and whom to address them \cite{jiang2018learning, kilinc2018multi, das2019tarmac, singh2018learning}.

A key difficulty, faced by a group of learning agents in such domains, is the need to efficiently exploit the available communication resources, such as limited bandwidth. The limited bandwidth exists in two processes of transmission: from agents to the scheduler and from the scheduler to agents as shown in Fig.~\ref{fig:arch}. This problem has recently attracted attention and one strategy has been proposed for limited bandwidth settings: downsizing the communication group via a scheduler \citep{zhang2013coordinating, kim2019learning, mao2019learning}. The scheduler allows a part of agents to communicate so that the bandwidth is not overwhelmed with all agents' messages. 
However, these methods limit the number of agents who can communicate instead of the communication content. Agents may share redundant messages which are unsustainable under bandwidth limitations. For example, a single large message can occupy the whole bandwidth. Also, these methods need specific configurations such as a predefined scale of agents' communication group \citep{zhang2013coordinating,kim2019learning} or a predefined threshold for muting agents \citep{mao2019learning}. Such manual configuration would be of a definite detriment in complex multi-agent domains. 

In this paper, we address the limited bandwidth problem by compressing the communication messages. First, from the perspective of communication theory, we view the messages as random vectors and prove that a limited bandwidth can be translated into a constraint on the communicated message entropy. Thus, agents should generate low-entropy messages to satisfy the limited bandwidth constraint. In more details, derived from source coding theorem \cite{shannon1948mathematical} and Nyquist criterion \cite{freeman2004telecommunication}, we state that in a noiseless channel, when a $K$-ary, bandwidth $B$, quantization interval $\Delta$ communication system transmits $n$ messages of dimension $d$ per second, the entropy of the messages $H(\boldsymbol{m})$ is limited by the bandwidth according to $H(\boldsymbol{m}) \leq \frac{2 B \log_{2} K}{n}+d\log_2{\Delta}$.

Moreover, to allow agents to send and receive low-entropy messages with useful and necessary information, we consider the problem of learning communication protocols and learning scheduling. Inspired by the variational information bottleneck method \citep{tishby2000information, alemi2016deep}, we propose a regularization method for learning informative communication protocols, named {\it Informative Multi-Agent Communication} (IMAC). Specifically, IMAC applies the variational information bottleneck to the communication protocol by viewing the messages as latent variables and approximating its posterior distribution. By regularizing the mutual information between the protocol's inputs (the input features extracted from agents) and the protocol's outputs (the messages), we learn informative communication protocols, which convey low-entropy and useful messages. Also, by viewing the scheduler as a virtual agent, we learn a weight-based scheduler with the same principle which aggregates compact messages by reweighting all agents' messages.

We conduct extensive experiments in different environments: cooperative navigation, predator-prey and StarCraftII. Results show that IMAC can convey low-entropy messages, enable effective communication among agents under the limited bandwidth constraint, and lead to faster convergence as compared with various baselines.

\section{Related Work}

Our work is related to
prior works in multi-agent reinforcement learning with communication, which mainly focus on two basic problems: who/whom and what to communicate. They are also expressed as the problem of learning scheduling and communication protocols.
One line of scheduling methods is to utilize specific networks to learn a weight-based scheduler by reweighting agents' messages, such as  
bi-direction RNNs in BiCNet \cite{peng2017multiagent}, a self-attention layer in TarMAC \cite{das2019tarmac}. 
Another line is to introduce various gating mechanisms to determine the groups of communication agents \cite{jiang2018learning, singh2018learning, kim2019learning, kilinc2018multi, mao2019learning}. 
Communication protocols are often learned in an end-to-end manner with a specific scheduler: from perceptual input (e.g., pixels) to communication symbols (discrete or continuous) to actions (e.g.,
navigating in an environment) \cite{foerster2016learning, kim2019learning}. While some works for learning the communication protocols focus on discrete human-interpretable communication symbols \cite{lazaridou2016multi, mordatch2018emergence},
our method learns a continuous communication protocol in an implicit manner \cite{foerster2016learning, sukhbaatar2016learning, jiang2018learning, singh2018learning}.

Methods for addressing the limited bandwidth problem are explored, such as downsizing the communication group via a scheduler. However, all scheduling methods suffer from content redundancy, which is unsustainable under bandwidth limitations. Even if only a single pair of agents is allowed to communicate, a large message may fail to be conveyed due to the limited bandwidth. In addition, scheduling methods with gating mechanisms are inflexible because they introduce manual configuration, such as the predefined size of a communication group \cite{zhang2013coordinating, kim2019learning}, or a handcrafted threshold for muting agents \cite{jiang2018learning, mao2019learning}. Moreover, most methods for learning communication protocols fail to compress the protocols and extract valuable information for cooperation \cite{jiang2018learning}. In this paper, we study the limited bandwidth in the aspect of communication protocols. Also, our methods can be extended into the scheduling if we utilize a weight-based scheduler.

\begin{figure}[t]
\centering
\includegraphics[width=\linewidth]{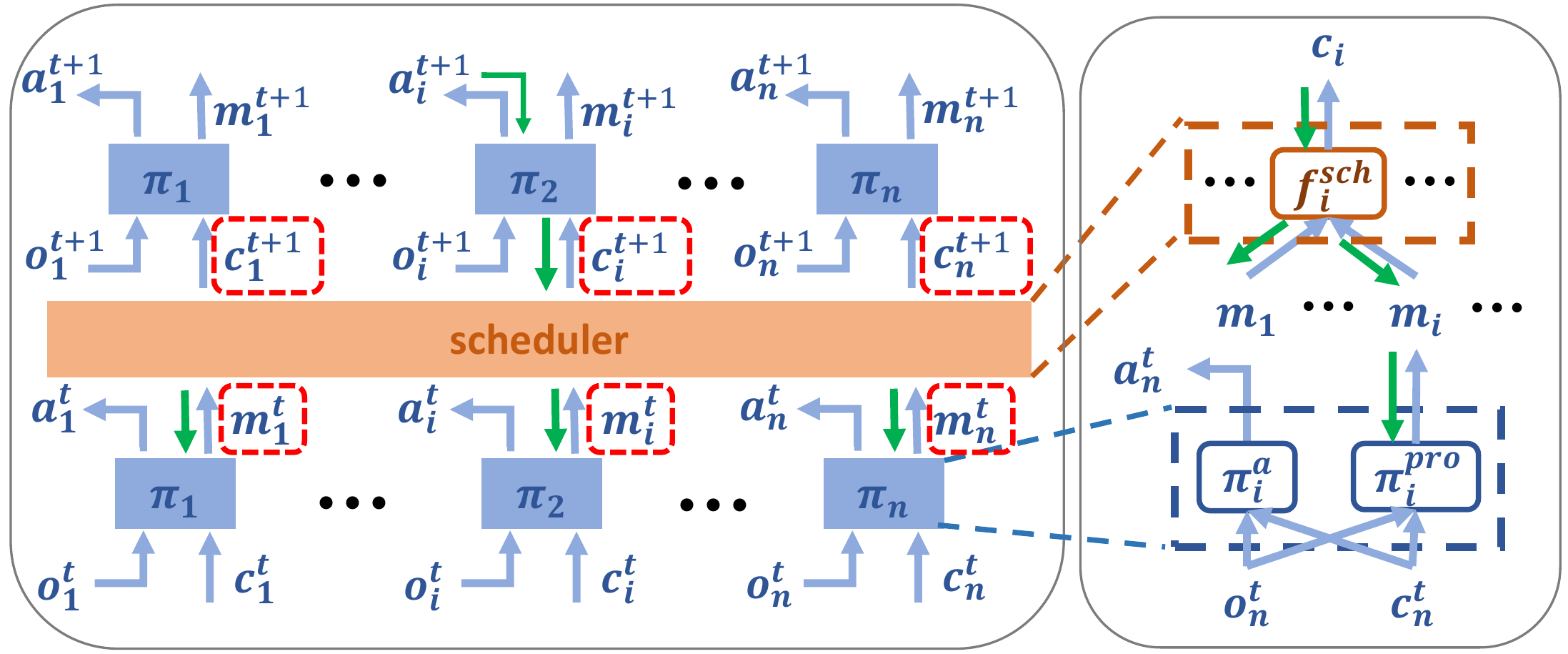}
\vspace{-20pt}
\caption{The Architecture of IMAC. \textbf{Left:} Overview of the communication scheme. The red dashed box means the communication process with a limited bandwidth constraint. The green line means the gradient flows. \textbf{Right:} The upper one is the scheduler for agent $i$. The below one is the policy $\pi^a_i$ and the communication protocol network $\pi^{pro}_i$ for agent $i$}.
\label{fig:arch}
\vspace{-20pt}
\end{figure} 

\section{Problem Setting}

We consider a communicative multi-agent reinforcement learning task, which is extended from Dec-POMDP and described as a tuple $\left \langle n, \boldsymbol{S}, \boldsymbol{A}, r, P, \boldsymbol{O}, \boldsymbol{\Omega}, \boldsymbol{M}, \gamma \right \rangle$, where $n$ represents the number of agents. $\boldsymbol{S}$ represents the space of global states. $\boldsymbol{A} = \{A_i\}_{i=1,\cdots,n}$ denotes the space of actions of all agents. $\boldsymbol{O} = \{O_i\}_{i=1,\cdots,n}$ denotes the space of observations of all agents. $\boldsymbol{M}$ represents the space of messages. $P : \boldsymbol{S} \times \boldsymbol{A} \to \boldsymbol{S}$ denotes the state transition probability function. All agents share the same reward as a function of the states and agents' actions $r : \boldsymbol{S} \times \boldsymbol{A} \to \mathbb{R}$. Each agent $i$ receives a private observation $o_i \in O_i$ according to the observation function $ \boldsymbol{\Omega}(s, i) : \boldsymbol{S} \to O_i$. $ \gamma \in [0,1] $ denotes the discount factor. As shown in Fig. \ref{fig:arch}, each agent receives observation $o_i$ and a scheduling message $c_i$, then outputs an action $a_i$ and a message $m_i$. A scheduler $f^{sch}_i$ is introduced to receive messages $[m_1, \cdots, m_n] \in \boldsymbol{M}$ from all agents and dispatch scheduling messages $c_i = f^{sch}_i(m_1, \cdots, m_n) \in M_i$ for each agent $i$.

We adopt a centralized training and decentralized execution paradigm \cite{lowe2017multi}, and further relax it by allowing agents to communicate. That is, during training, agents are granted access to the states and actions of other agents for the centralized critic, while decentralized execution only requires individual states and scheduled messages from a well-trained scheduler. 

Our end-to-end method is to learn a communication protocol $\pi^{pro}_i(m_i|o_i, c_i)$, an policy $\pi^a_i(a_i|o_i, c_i)$, and a scheduler $f^{sch}_i(c_i|m_1, \cdots, m_n)$, which jointly maximize the expected discounted return for each agent $i$:
\begin{equation}
\begin{small}
\begin{aligned}
J_i &= \mathbb{E}_{\pi^a_i, \pi^{pro}_i, f^{sch}_i}[\Sigma_{t=0}^\infty \gamma^t r_i^t(s,a)] \\& = \mathbb{E}_{\pi^a_i, \pi^{pro}_i, f^{sch}_i}[Q_i(s,a_1, \cdots, a_n)] \\& \approx \mathbb{E}_{\pi^a_i, \pi^{pro}_i, f^{sch}_i}[Q_i(o_1, \cdots, o_n, c_1, \cdots, c_n,a_1, \cdots, a_n)] \\& = \mathbb{E}_{\pi^a_i, \pi^{pro}_i, f^{sch}_i}[Q_i(h_1, \cdots, h_n, a_1, \cdots, a_n)]
\end{aligned}
\end{small}
\end{equation}\normalsize
where $r^t_i$ is the reward received by the $i-$th agent at time $t$, $Q_i$ is the centralized action-value function for each agent $i$, and $\mathbb{E}_{\pi^a_i, \pi^{pro}_i, f^{sch}_i}$ denotes an expectation over the trajectories $\langle s, a_i, m_i, c_i \rangle$ generated by $p^{\pi^a_i}, \pi^a_i(a_i|o_i, c_i), \pi^{pro}_i(m_i|o_i, c_i), f^{sch}_i(c_i|m_1, \cdots, m_n)$. Here we follow the simplification in \cite{lowe2017multi} to replace the global states with joint observations and use an abbreviation $h_i$ to represent the joint value of $[o_i, c_i]$ in the rest of this paper.

The limited bandwidth $B$ is a range of frequencies within a given band. It exists in the two processes of transmission: messages from agents to the scheduler and scheduling messages from the scheduler to agents as shown in Fig.~\ref{fig:arch}. In the next section, we will discuss how the limited bandwidth $B$ affects the communication.


\section{Connection between Limited Bandwidth and Multi-agent Communication}

In this section, from the perspective of communication theory, we show how the limited bandwidth $B$ requires low-entropy messages throughout the transmission. We then discuss how to measure the message's entropy.



\subsection{Communication Process}

We show the communication process (Figure \ref{fig:Communication Scheme}) from agents to the scheduler, which consists of five stages: analog-to-digital, coding, transmission, decoding and digital-to-analog \cite{freeman2004telecommunication}. When agent transmits a continuous message $m_i$  of agent $i$, an analog-to-digital converter (ADC) maps it into a countable set. An ADC can be modeled as two processes: sampling and quantization. Sampling converts a time-varying signal into a sequence of \textit{discrete-time} real-value signal. This operation is corresponding to the discrete timestep in RL. Quantization replaces each real number with an approximation from a finite set of \textit{discrete values}. In the coding phase, the quantized messages $m^{\Delta}_i$ is mapped to a bitstream using source coding methods such as Huffman coding. In the transmission phase, the transmitter modulates the bitstream into wave, and transmit the wave through a channel, then the receiver demodulates the wave into another bitstream due to some distortion in the channel. Then, decoding is the inverse operation of coding, mapping the bitstream to the recovered messages $\hat{m}^{\Delta}_i$. Finally, the scheduler receives a reconstructed analog message from a digital-to-analog converter (DAC). The same process happens when sending the scheduled messages $c_i$ from the scheduler to the agent $i$. The bandwidth actually restricts the transmission phase.
\begin{figure}[!htbp]
\centering
\includegraphics[width=\linewidth]{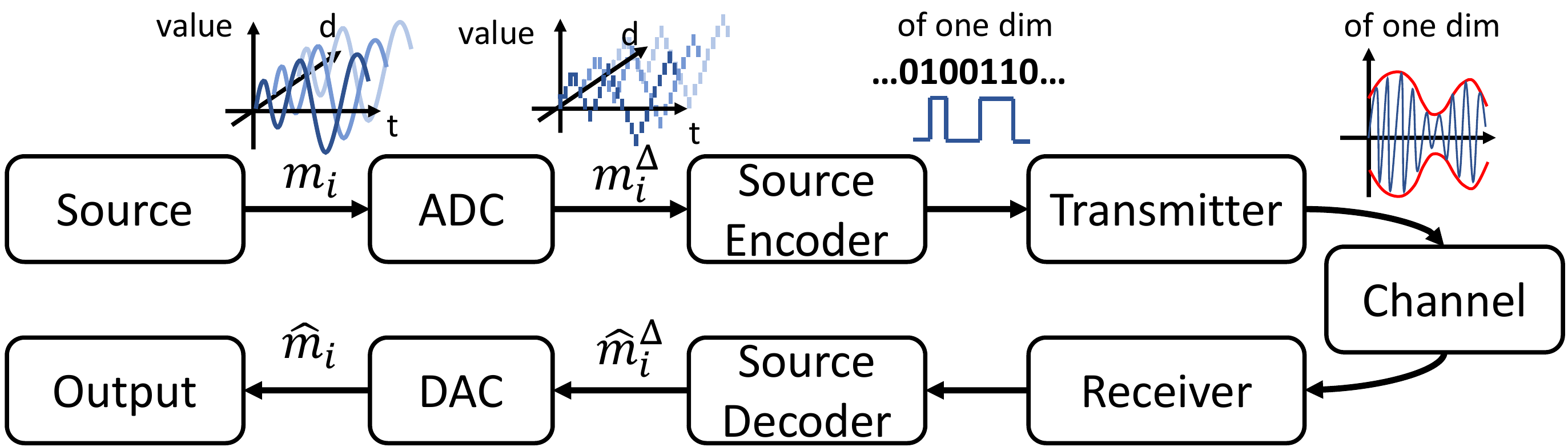}
\vspace{-20pt}
\caption{Overview of the Communication Process. Axes of messages are time, dimension of the message vector, and value of each element in the vector.}
\label{fig:Communication Scheme}
\end{figure}

\vspace{-10pt}

\subsection{Limited Bandwidth Restricts Message's Entropy} \label{link}

We model the messages $m_i$ as continuous random vectors $M_i$, i.e., continuous vectors sampled from a certain distribution. The reason is that a message is sent by one agent in each timestep, while over a long duration, the messages follow some distributions. We abuse $\boldsymbol{m}$ to represent the random vector by omitting the subscript, and explain the subscript in special cases. 


For simplicity, we consider sending an element $X$ of the continuous random vector $\boldsymbol{m}$, which is a continuous random variable, and then extend our conclusion to $\boldsymbol{m}$. First, we quantize the continuous variable into discrete symbols. The quantization brings a gap between the entropy of the discrete variables and differential entropy of the continuous variables.

\begin{remark}[Relationship between entropy and differential entropy]
\label{remark:differential entropy}
Consider a continuous random variable $X$ with a probability density function $f_X(x)$. This
function is then quantized by dividing its range into $K$ levels of interval $\Delta$, where $K= ceil(|X|/\Delta)$, and $|X|$ is max amplitude of variable. The quantized variable is $X^{\Delta}$. Then the difference between differential entropy $H(X)$ and entropy $H\left(X^{\Delta}\right)$ is $H(X)-H\left(X^{\Delta}\right)=\log _{2}(\Delta)$.
\end{remark}

Note that for a fixed small identical interval $\Delta$, there is only a constant difference between differential entropy and entropy. Then we encode these quantized symbols.

\begin{remark}[Source Coding Theorem \cite{shannon1948mathematical}]
\label{remark:source_coding}
In the source coding phase, a set of $n$ quantized symbols is to be encoded into bitstreams. These symbols can be treated as $n$ independent samples of a discrete random variable $X^{\Delta}$ with entropy $H(X^{\Delta})$. Let $L$ be the average number of bits to encode the $n$ symbols. The minimum $L$ satisfies $H(X^{\Delta}) \le L < H(X^{\Delta})+\frac{1}{n}$.
\end{remark}

Remark \ref{remark:source_coding} regularizes the coding phase in the communication process.
Then in the transmission, over a noiseless channel, the maximum bandwidth is defined as the maximum data rate:

\begin{remark}[The Maximum Data Rate \cite{freeman2004telecommunication}]
\label{remark:rate}
The maximum data rate $R_{max}$ (bits per second) over a noiseless channel satisfies:
$R_{max} = 2 B \log_{2} K$, where $B$ is the bandwidth (Hz) and $K$ is the number of signal levels.
\end{remark}

Remark \ref{remark:rate} is derived from the Nyquist criterion \cite{freeman2004telecommunication} and specifies how the bandwidth of a communication system affects the transmission data rate for reliable transmission in the noiseless condition. Based on these three remarks, we show how the limited bandwidth constraint affects the multi-agent communication.

\begin{prop}
\label{prop:entropy1}
In a noiseless channel, the bandwidth of channel $B$ limits the entropy of the messages' elements.
\end{prop}

\vspace{-10pt}

\begin{proof}[Proof]
Given a message's element $X$ as an i.i.d continuous random variable with differential entropy $H(X)$, its quantized time series $X^{\Delta}_1, \cdots, X^{\Delta}_t, \cdots$ (here the subscript means different times) with entropy $H(X^{\Delta})=H(X)-\log_2{\Delta}$, the communication system's bandwidth $B$, as well as the signal levels $K$, the communication system transmits $n$ symbols per second. So the transmission rate $R_{trans} (\frac{\text{bit}}{\text{second}})=  n \cdot R_{code} (\frac{\text{bit}}{\text{symbol}}) \geq n \cdot H(X^{\Delta}) \geq n \cdot (H(X)-\log{\Delta})$.\footnote{$\frac{\text{bit}}{\text{second}}$ and $\frac{\text{bit}}{\text{symbol}}$ are units of measure.} According to Remark \ref{remark:rate}, $R_{trans} \leq R_{max} = 2 B \log_{2} K$. Consequently, we have $H(X) \leq \frac{2 B \log_{2} K}{n}+\log_2{\Delta}$. 
\end{proof}

\begin{prop}
\label{prop:entropy2}
In a noiseless channel, the bandwidth of channel $B$ limits the entropy of the messages.
\end{prop}

\vspace{-10pt}

\begin{proof}[Proof]
When sending the random vector, i.e., the message $M_i=[X_1, X_2, \cdots, X_d]$, where the subscript means different entries of the vector and $d$ is the dimension, each variable $X_i$ occupies a bandwidth $B_i$, which satisfies $\sum_{i=1}^{d}{B_i}=B$. Assume all entries are quantilized with the same interval, according to \cite{cover2012elements}, the upper bound of the messages $H(M_i) = H(X_1, \cdots, X_d) \leq \sum_{i=1}^{d}{H(X_i)} \leq \frac{2 B \log_{2} K}{n}+d\log_2{\Delta}$. 
\end{proof}

Eventually, a limited bandwidth $B$ enforces an upper bound $H_c$ to the message's entropy $H(M_i) \leq H_c \propto B$. 

\subsection{Measurement of a Message's Entropy} \label{measure}

The messages $M_i$ as an i.i.d random vector can follow any distribution, so it is hard to determine the message's entropy. So, we keep a historical record of the message and find a quantity to measure the message's entropy.

\begin{prop}
\label{prop:measure}
When we have a historical record of the messages to estimate the messages' mean $\boldsymbol{\mu}$ and co-variance $\boldsymbol{\boldsymbol{\Sigma}}$, the entropy of the messages is upper bounded by $\frac{1}{2}\log((2\pi e)^d |\boldsymbol{\Sigma}|)$, where $d$ is the dimension of $M_i$.
\end{prop}

\begin{proof}[Proof]

The message $M_i$ follows a certain distribution, and we are only certain about its mean and variance. According to the principle of maximum entropy \cite{jaynes1957information}, the Gaussian distribution has maximum entropy relative to all probability distributions covering the entire real line $(-\infty,\infty)$ but having a finite mean and finite variance (see the proof in \cite{cover2012elements}). So $H(M_i) \leq \frac{1}{2}\log((2\pi e)^d \boldsymbol{\Sigma})$, where the right term is the entropy of multivariate Gaussian $N(\boldsymbol{\mu},\boldsymbol{\Sigma})$.
\end{proof}

We conclude that $\frac{1}{2}\log((2\pi e)^d |\boldsymbol{\Sigma}|)$ offers an upper bound to approximate $H(M_i)$, and this upper bound should be less than or equal to the limited bandwidth constraint to ensure that the message with any possible distribution satisfies the limited bandwidth constraint.


\section{Informative Multi-agent Communication}

As shown in the previous section, the limited bandwidth requires agents to send low-entropy messages. In this section, we first introduce our method for learning a valuable and low-entropy communication protocol via the information bottleneck principle. Then, we discuss how to use the same principle in scheduling.

\subsection{Variational Information Bottleneck for Learning Protocols}
\label{vib}
We propose an informative multi-agent communication via information bottleneck principle to learn protocols. Concretely, we propose an information-theoretic regularization on the mutual information $I(H_i; M_i)$ between the messages and the input features
\begin{small}
\begin{align}
I(H_i; M_i) \leq I_c
\end{align}
\end{small}\normalsize
where $M_i$ is a random vector with a probability density function $p_{M_i}(m_i)$, which represents the possible assignments of the messages $m_i$, and $H_i$ is a random vector with a probability density function $p_{H_i}(h_i)$ which the possible values of $[o_i, c_i]$. We omit the subscripts in the density functions in the rest of the paper. Eventually, with the help of variational information bottleneck \cite{alemi2016deep}, this regularization enforces agents to send low-entropy messages.


Consider a scenario with $n$ agents' policies $\{\pi^a_i\}_{i=1,\cdots,n}$ and protocols $\{\pi^{pro}_i\}_{i=1,\cdots,n}$ which are parameterized by $\{\theta_i\}_{i=1,\cdots,n} = \{\theta^a_i, \theta^{pro}_i\}_{i=1,\cdots,n}$, and with schedulers $\{f^{sch}_i\}_{i=1,\cdots,n}$ which is parameterized by $\{\phi_i\}_{i=1,\cdots,n}$. 
Consequently, for learning the communication protocol with fixed schedulers, the agent $i$ is supposed to maximize:
\begin{small}
\begin{align*}
J ( \theta_i ) &= \mathbb{E}_{\pi^a_i, \pi^{pro}_i, f^{sch}_i}[Q_i(h_1, \cdots, h_n,a_1,\cdots,a_n)] \\& \qquad \text{s.t.} \ \ I(H_i; M_i) \leq I_c
\end{align*}
\end{small}\normalsize
Practically, we propose to maximize the following objective using the information bottleneck Lagrangian:
\begin{small}
\begin{multline}
J' ( \theta_i ) = \mathbb{E}_{\pi^a_i, \pi^{pro}_i, f^{sch}_i}[Q_i(h_1, \cdots, h_n,a_1, \cdots, a_n)] \\ - \beta I(H_i; M_i)
\end{multline}
\end{small}\normalsize
where the $\beta$ is the Lagrange multiplier. The mutual information is defined according to:
\begin{small}
\begin{align*}
\begin{aligned}
I(H_i; M_i) &= \iint p ( {h_i} , {m_i} ) \log \frac { p ( {h_i} , {m_i} ) } { p ( {h_i}) p ( {m_i} ) } \text{d}h_i \text{d}m_i \\&= \iint p ( {h_i} ) \pi^{pro} ( {m_i} | {h_i}) \log \frac { \pi^{pro} ( {m_i} | {h_i} ) } { p ( {m_i} ) } \text{d}h_i \text{d}m_i
\end{aligned}
\end{align*}
\end{small}\normalsize
where $p({h_i}, {m_i})$ is the joint probability of ${h_i}$ and ${m_i}$. 

However, computing the marginal distribution $p({m_i})=\int \pi^{pro}({m_i}|{h_i})p({h_i}) d{h_i}$ can be challenging since we do not know the prior distribution of $p({h_i})$. 
With the help of variational information bottleneck \cite{alemi2016deep}, we use a Gaussian approximation $z({m_i})$ of the marginal distribution $p({m_i})$ and view $\pi^{pro}$ as multivariate variational encoders. Since $D_{KL}[p({m_i})||z({m_i})] \geq 0$, where the $D_{KL}$ is the Kullback-Leibler divergence, we expand the KL term and get $\int p({m_i}) \log p({m_i}) d{m_i} \geq  \int p({m_i}) \log z({m_i}) d{m_i}$, an upper bound on the mutual information $I(H_i; M_i)$ can be obtained via the KL divergence:
\begin{small}
\begin{equation}
\begin{aligned}
I(H_i; M_i) &\leq \int p({h_i}) \pi^{pro}_i({m_i}|{h_i}) \log \frac{\pi^{pro}_i({m_i} | {h_i})} {z({m_i})} d{h_i} d{m_i} \\&= \mathbb{E}_{{h_i}\sim p({h_i})}[D_{KL}[\pi^{pro}({m_i}|{h_i})\|z({m_i})]]
\end{aligned}
\end{equation}
\end{small}\normalsize
This provides a lower bound $\tilde{J}(\theta)$ on the regularized objective that we maximize:
\begin{small}
\begin{align*}
\begin{aligned}
J'(\theta_i) \geq \tilde{J}(\theta_i)
= \mathbb{E}_{\pi^a_i, \pi^{pro}_i, f^{sch}_i}[Q_i(h_1, \cdots, h_n ,a_1, \cdots, a_n)] \\- \beta  \mathbb{E}_{{h_i} \sim p({h_i})}[D_{KL}[\pi^{pro}_i({m_i}|{h_i})\|z({m_i})]]
\end{aligned}
\end{align*}
\end{small}\normalsize
Consequently, the objective's derivative is:
\begin{small}
\begin{multline}
\nabla_{\theta_i} \tilde{J}\left(\pi_i\right) = \mathbb{E}_{\pi^a_i, \pi^{pro}_i, f^{sch}_i} \Big[\nabla_{\theta_i} \log \left(\pi_i\left(a_{t}|s_{t}\right)\right) \\ Q_i(h_1, \cdots, h_n, a_1, \cdots, a_n) - \beta \nabla_{\theta_i}D_{KL}[\pi^{pro}({m_i}|{h_i})\|z({m_i})] \Big]
\end{multline}
\end{small}\normalsize
With the variational information bottleneck, we can control the messages' distribution and thus control their entropy with different prior $z({m_i})$ to satisfy different limited bandwidths in the training stage. That is, with the regulation of $D_{KL}[p({m_i}|{h_i})\|z({m_i})]$, the messages' probability density function $p({m_i}) =  \sum_{{h_i}} p({m_i} | {h_i}) p({h_i}) \approx \sum_{{h_i}} z({m_i}) p({h_i}) = z({m_i}) \sum_{{h_i}} p({h_i})=z({m_i})$. Thus $H(M_i) = - \int p \log{p} \text{d}m_i \approx - \int z \log{z} \text{d}m_i$.

\subsection{Unification of Learning Protocols and Scheduling}

The scheduler for agent $i$ is $f^{sch}_i(c_i|m_1, \cdots, m_n)$. Recall the communication protocols for agent $i$: $\pi^{pro}_i(m_i|h_i)$. Due to the same form of the protocol and the scheduling, the scheduler is supposed to follow the same principle for learning a weight-based mechanism. Variational information bottleneck can be applied in scheduling for agent $i$ with regularization on the mutual information between the scheduling messages ${c_i}$ and all agents' messages $I({C_i}; {M_1}, \cdots, {M_n})$, where $C_i$ is a random vector with a probability density function $p_{C_i}(c_i)$, which represent different values of $c_i$. We follow the joint training scheme for training the communication protocol and scheduling \cite{foerster2016learning}, which allows the gradients flow across agents from the recipient agent to the scheduler to the sender agent.

\subsection{Implementation of the limited bandwidth Constraint} 

During the execution stage, the messages may still obey the low-entropy requirement. We implement the limited bandwidth during the execution according to the low-entropy principle. Concretely, we use a batch-normalization-like layer which records the messages' mean and variance during training as Prop. \ref{prop:entropy2} requires. And it clips the messages by reducing their variance during execution. The purpose of our normalization layer is to measure the messages' mean and variance, as well as to simulate the external limited bandwidth constraint during execution. It is customized and different from standard batch normalization \cite{ioffe2015batch}.  

\begin{figure*}[t]
 
\centering
\includegraphics[width=\linewidth]{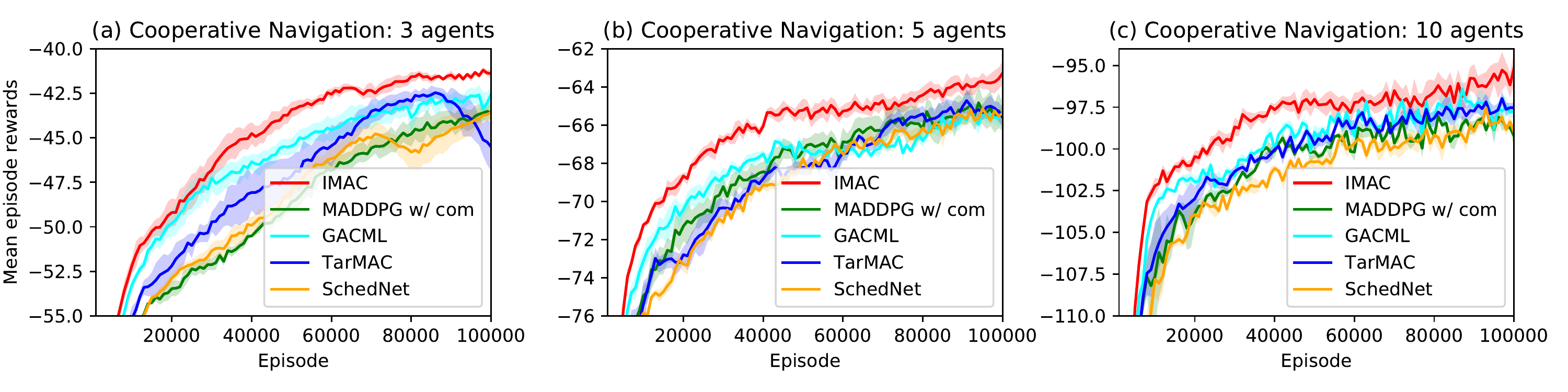}
\vspace{-25pt}
\caption{Learning curves comparing IMAC to other methods for cooperative navigation. As the number of agents increases (from left to right), IMAC improves agents’ performance and converge faster.}
\label{fig:co_na}
\end{figure*} 

\begin{figure*}[t]
\centering
\includegraphics[width=\linewidth]{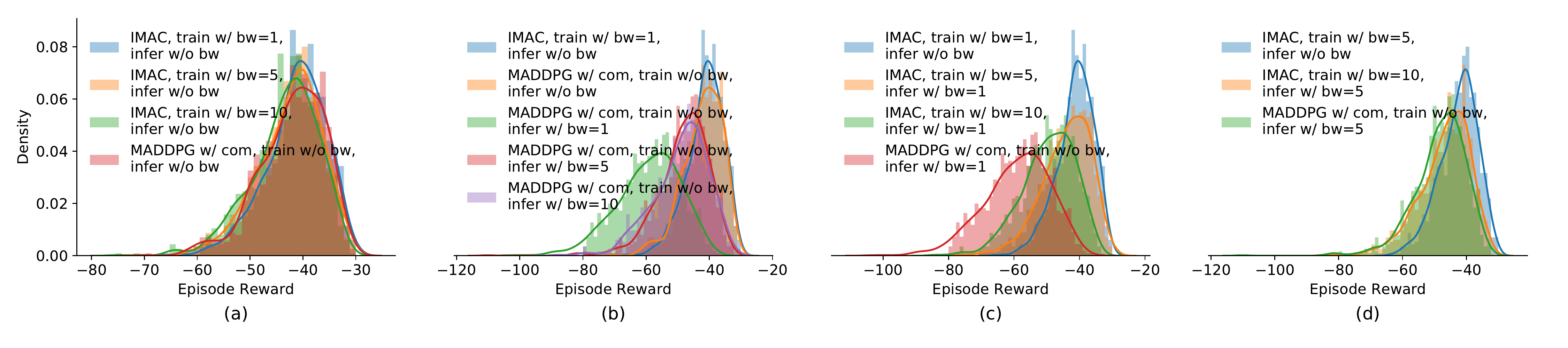}
\vspace{-25pt}
\caption{Density plot of episode reward per agent during the execution stage. (a) Reward distribution of IMAC trained with different prior distributions against MADDPG with communication. (b) Reward distribution of MADDPG with communication under different limited bandwidth environment. (c), (d) Reward distribution of IMAC trained with different prior distributions against MADDPG with communication under the same bandwidth constraint. ``bw=$\delta$" means in the implementation of the limited bandwidth constraint, the variance $\boldsymbol{\Sigma}$ of Gaussian distribution is $\delta$.}
\label{fig:limited_bandwidth}
\end{figure*}

\section{Experiment}

\textbf{Environment Setup.} We evaluate IMAC on a variety of tasks and environments: cooperative navigation and predator prey in Multi Particle Environment \cite{lowe2017multi}, as well as 3m and 8m in StarCraftII \cite{samvelyan19smac}. The detailed experimental environments are elaborated in the following subsections as well as in supplementary material.

\textbf{Baselines.} We choose the following methods as baselines: (1) TarMAC \cite{das2019tarmac}: A state-of-the-art multi-agent communication method for limited bandwidth, which uses a self-attention weight-based scheduling mechanism for scheduling and learns the communication protocol in an end-to-end manner. (2) GACML \cite{mao2019learning}: A multi-agent communication method for limited bandwidth, which uses a gating mechanism for downsizing communication agents and learns the communication protocol in an end-to-end manner. (3) SchedNet \cite{kim2019learning}: A multi-agent communication method for limited bandwidth, which uses a selecting mechanism for downsizing communication agents and learns the communication protocol where the message is one real value (float16 type). Also, we modify MADDPG \cite{lowe2017multi} and QMIX \cite{rashid2018QMIX} with communication as baselines to show that IMAC can facilitate different multi-agent methods and work well with limited bandwidth constraints.

\subsection{Cooperative Navigation}

In this scenario, $n$ agents cooperatively reach $k$ landmarks while avoiding collisions. Agents observe the relative positions of other agents and landmarks and are rewarded with a shared credit based on the sum of distances between agents to their nearest landmark, while they are penalized when colliding with other agents. Agents learn to infer and occupy the landmarks without colliding with other agents based on their own observation and received information from other agents.

\textbf{Comparison with baselines.} We compare IMAC with TarMAC, GACML, and SchedNet because they represent the method of learning the communication protocols via end-to-end training with the specific scheduler and clipping the messages respectively. Also due to different bandwidth definitions, we also compare with the modified MADDPG with communication, which is trained without the limited bandwidth constraint, because it offers the baseline performance of the centralized training and decentralized execution.

Figure \ref{fig:co_na} (a) shows the learning curve of 100,000 episodes in terms of the mean episode reward over a sliding window of 1000 episodes. We can see that at the end of the training, agents trained with communication have higher mean episode reward. According to \cite{lowe2019pitfalls}, ``increase in reward when adding a communication channel" is sufficient to effective communication. Additionally, IMAC outperforms other baselines along the process of training, i.e., IMAC can reach upper bound of performance early. By using the information bottleneck method, messages are less redundant, thus agents converge fast (More analysis can be seen in the supplementary materials). 

\begin{table*}[t]
\centering
\begin{tabular}{|c|c|c|c|c|c|}
\hline
\diaghead(-4,1){aaaaaaaaaaaaaaa}%
{Predator}{Prey} & IMAC                                          & TarMAC                                 & GACML                                   & SchedNet                                & MADDPG w/ com                            \\ 
\hline
IMAC                              & \textbf{32.32}\textbackslash{}\textbf{-4.26}  & \textbf{28.91}\textbackslash{} -22.27  & \textbf{28.25} \textbackslash{} -26.11  & 22.67 \textbackslash{} -36.53           & \textbf{34.33} \textbackslash{} -22.62   \\ 
\hline
TarMAC                            & 25.13 \textbackslash{} \textbf{-2.94}         & 23.45 \textbackslash{} -20.42          & 22.12 \textbackslash{} -16.51           & \textbf{32.52} \textbackslash{} -42.39  & 27.54 \textbackslash{} -29.36            \\ 
\hline
GACML                             & 21.52 \textbackslash{}\textbf{-12.74}         & 11.49 \textbackslash{} -24.93          & 13.93 \textbackslash{} -12.95           & 25.49 \textbackslash{} -27.42           & 28.47 \textbackslash{} -27.75            \\ 
\hline
SchedNet                          & 24.74 \textbackslash{}\textbf{-9.63}          & 7.84 \textbackslash{} -23.56           & 12.48 \textbackslash{} -23.67           & 5.98 \textbackslash{} -26.82            & 21.53 \textbackslash{} -26.43            \\ 
\hline
MADDPG w/ com                     & 28.63 \textbackslash{} -15.60                 & 19.32 \textbackslash{} -21.52          & 26.91 \textbackslash{} -19.76           & 22.17 \textbackslash{} -35.37           & 16.87 \textbackslash{} \textbf{-13.09}   \\
\hline
\end{tabular}
\caption{Cross-comparison between IMAC and baselines on predator-prey.}
\label{table:predator_prey}

\end{table*}

We also investigate the performance when the number of agents increases. We make a slight modification on environment about agents' observation. According to \cite{jiang2018learning}, we constrain that each agent can observe the nearest three agents and landmarks with relative positions and velocity. Figure \ref{fig:co_na} (b) and (c) show that the leading performance of IMAC in the 5 and 10-agent scenarios.

\textbf{Performance under stronger limited bandwidth.} We first train IMAC with different priors to satisfy different bandwidths. Then we evaluate IMAC and the modified MADDPG with communication by checking agents’ performance under different limited bandwidth constraints during the execution stage. Figure \ref{fig:limited_bandwidth} shows the density plot of episode reward per agent during the execution stage. We first respectively train IMAC with different prior distributions $z(M_i)$ of $N(0,1)$, $N(0,5)$, and $N(0, 10)$, to satisfy different default limited bandwidth constraints. Consequently, the entropy of  agents' messages satisfies the bandwidth constraints. 
In the execution stage, we constrain these algorithms into different bandwidths. As depicted in Figure \ref{fig:limited_bandwidth} (a), IMAC with different prior distributions can reach the same outcome as MADDPG with communication. Figure \ref{fig:limited_bandwidth} (b) shows that MADDPG with communication fails in the limited bandwidth environment. From Figure \ref{fig:limited_bandwidth} (c) and (d), we can see that the same bandwidth constraint is less effective in IMAC than MADDPG with communication. Results here demonstrate that IMAC discards useless information without impairment on performance. 

\textbf{Ablation.} We investigate the effect of limited bandwidth and $\beta$ on multi-agent communication on the performance of agents. Figure \ref{fig:ablation_new} (a) shows the learning curve of IMAC with different prior distributions. IMAC with $z(M_i)=N(0,1)$ achieves the best performance. When the variance is smaller or larger, the performance suffers some degradation. It is reasonable because a smaller variance means a more lossy compression, leading to less information sharing. A larger variance must bring about more redundant information than the variance without regulation, thus leading to slow convergence. $\beta$ controls the degree of compression between $h_i$ and $m_i$ for each agent $i$: the larger $\beta$, the more lossy compression. Figure \ref{fig:ablation_new} (b) shows a similar result to the ablation on limited bandwidth constraint. The reason is the same: a larger $\beta$ means a more strict compression while a smaller $\beta$ means a less strict one.

The ablation shows that as a compression algorithm, the information bottleneck method extracts the most informative elements from the source. A proper compression rate is good for multi-agent communication, because it cannot only avoid losing much information caused by higher compression, but also resist much noise caused by lower compression.

\vspace{-10pt}
\begin{figure}[htbp]
  \includegraphics[width=\linewidth]{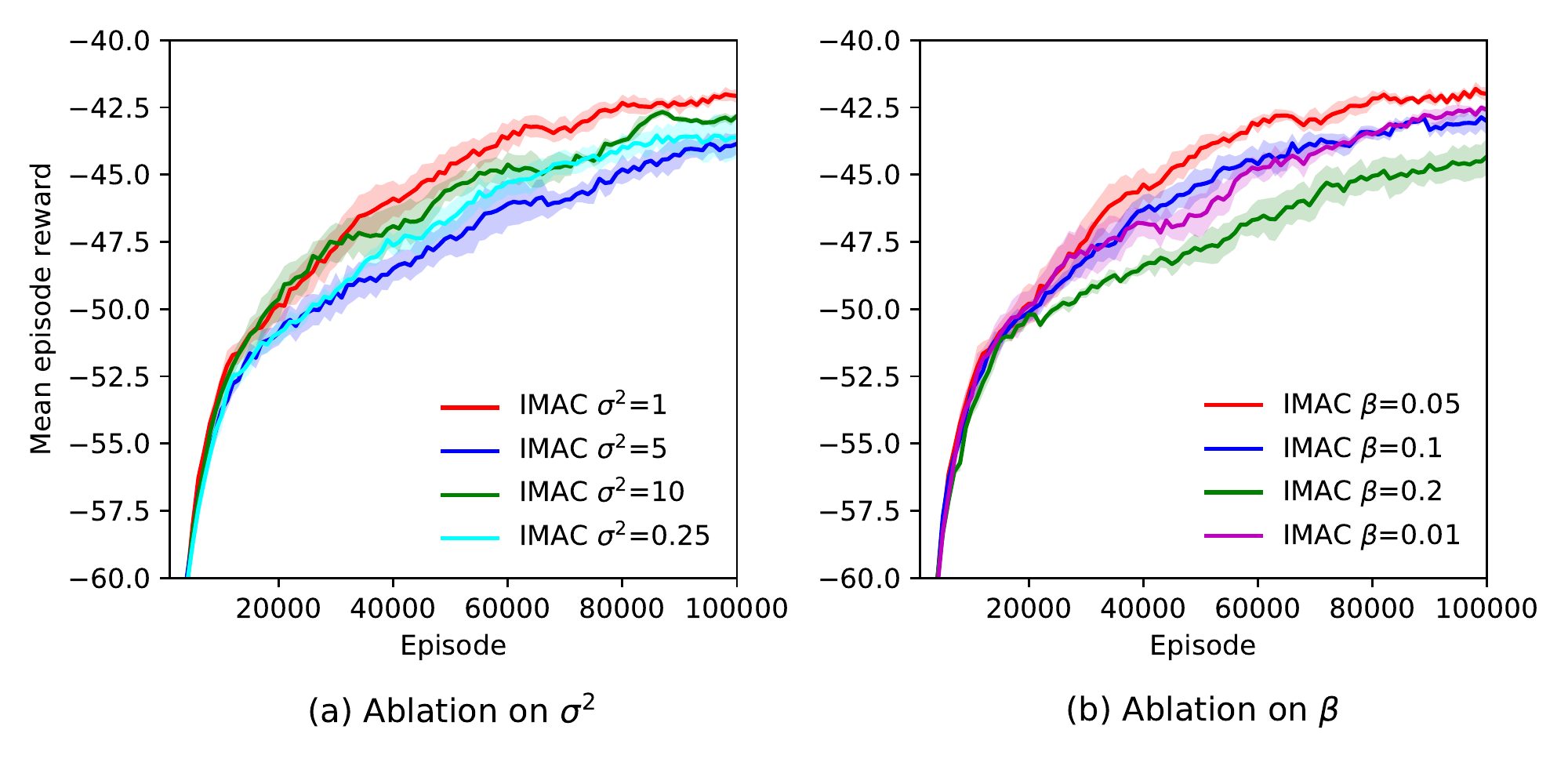}
  \centering
  \vspace{-25pt}
  \caption{Ablation: learning curves with respect to $\Sigma$ and $\beta$}
  \label{fig:ablation_new}
\end{figure}

\subsection{Predator Prey}

In this scenario, $m$ slower predators chase $n$ faster preys around an environment with $l$ landmarks impeding the way. As same as cooperative navigation, each agent observes the relative position of other agents and landmarks. Predators share common rewards, which are assigned based on the collision between predators and preys, as well as the minimal distance between two groups. Preys are penalized for running out of the boundary of the screen. In this way, predators would learn to approach and surround preys, while preys would learn to feint to save their teammates.

We set the number of predators as 4, the number of preys as 2, and the number of landmarks as 2. We use the same architecture and hyper-parameter as configuration in cooperative navigation. We train our agents by self-play for 100,000 episodes and then evaluate performance by cross-comparing between IMAC and the baselines. We average the episode rewards across 1000 rounds (episodes) as scores.

\begin{table*}[!ht]
{\begin{tabular}{|l|l|l|l|l|l|l|}
\hline
Predator \textbackslash Prey & MADDPG\_e1                      & MADDPG\_e5                       & IMAC                           & IMAC\_t5\_e1                   & IMAC\_t10\_e1                 & IMAC\_t10\_e5                  \\ \hline
MADDPG\_e1                   & 18.01 \textbackslash \textbf{-14.22}     & 24.15 \textbackslash  -29.88     & 22.38 \textbackslash -16.91    & 47.59 \textbackslash -45.64    & 34.25 \textbackslash -27.68   & 50.81 \textbackslash -43.62    \\ \hline
MADDPG\_e5                   & 26.32 \textbackslash -20.48     & 15.67 \textbackslash \textbf{-11.59}      & 29.06 \textbackslash -22.16    & 27.07 \textbackslash -22.89    & 23.44 \textbackslash -20.41   & 32.24 \textbackslash -26.46    \\ \hline
IMAC                         & \textbf{51.24} \textbackslash -42.56   & \textbf{37.37} \textbackslash -45.521     & \textbf{44.64} \textbackslash -36.49 & \textbf{49.12} \textbackslash -42.65   & \textbf{36.63} \textbackslash -30.03  & 35.42 \textbackslash \textbf{-28.82}   \\ \hline
IMAC\_t5\_e1                 & 38.86 \textbackslash -32.06   & 34.54 \textbackslash -35.03     & 9.97 \textbackslash \textbf{-3.11}     & 26.25 \textbackslash -21.06  & 11.80 \textbackslash -7.558   & \textbf{38.32} \textbackslash -32.28 \\ \hline
IMAC\_t10\_e1                & 26.67 \textbackslash -21.418    & 34.99 \textbackslash -35.02     & 9.71 \textbackslash \textbf{-4.11}   & 9.82 \textbackslash{}-6.92 & 9.82 \textbackslash -6.92 & 37.50 \textbackslash -31.30 \\ \hline
IMAC\_t10\_e5                & 45.88 \textbackslash -38.27 & 26.39 \textbackslash -35.42 & 11.51 \textbackslash \textbf{-9.12}  & 30.02 \textbackslash -27.41 & 29.08 \textbackslash -25.661  & 22.25 \textbackslash -16.51 \\ \hline 
\end{tabular}}
\vspace{-10pt}
\caption{Cross-comparison in different bandwidths on predator-prey. ``t5" means that IMAC is trained with the variance $|\Sigma|=5$. ``e1" means that during the execution, we use the batch-norm like layer to clip the message to enforce its variance $|\Sigma|=5$.}
\label{table:predator_prey:limited}
\end{table*} 

\textbf{Comparison with baselines.} We use the same baselines as in the cooperative navigation. Table \ref{table:predator_prey} represents the cross-comparing between IMAC and the baselines. Each cell consists of two numbers which denote the mean episode rewards of the predators and preys respectively. The larger the score is, the better the algorithm is. We first focus on the mean episode rewards of predator row by row. Facing the same prey, IMAC has higher scores than the predators of all the baselines and hence are stronger than other predators. Then, the mean episode rewards of the prey column by column show the ability of the prey to escape. We can see that IMAC has higher scores than the preys of most baselines and hence are stronger than other preys. We argue that IMAC leads to better cooperation than the baselines even in competitive environments and the learned policy of IMAC predators and preys can generalize to the opponents with different policies.

\textbf{Performance under stronger limited bandwidth.} Similar to the cooperative navigation, we evaluate algorithms by showing the performance under different limited bandwidth constraints during execution. Table \ref{table:predator_prey:limited} shows the performance under different limited bandwidth constraints during inference in the environment of predator and prey. We can see with limited bandwidth constraint, MADDPG with communication and IMAC suffer a degradation of performance. However, IMAC outperforms MADDPG with communication with respect to resistance to the effect of limited bandwidth. 

\subsection{StarCraftII}

We apply our method and baselines to decentralized StarCraft II micromanagement benchmark to show that IMAC can facilitate different multi-agent methods. We use the
setup introduced by SMAC \cite{samvelyan19smac} and consider combat scenarios.

\textbf{3m and 8m.} Both tasks are symmetric battle scenarios, where marines controlled by the learned agents try to beat enemy units controlled by the built-in game AI. Agents will receive some positive (negative) rewards after having enemy (allied) units killed and/or a positive (negative) bonus for winning (losing) the battle.

\textbf{Comparison with Baselines.} We adapt QMIX with communication and with IMAC, because QMIX uses the centralized training decentralized execution scheme for discrete actions. We also evaluate MADDPG with communication. However, SMAC is a discrete-action scenario, while MADDPG is for continuous control. Even if we modify the MADDPG into discrete action setup, it still fails to get any positive reward. Fig.~\ref{fig:QMIX} shows the learning curve of 200 episodes in terms of the mean episode rewards. We can see that at the beginning, QMIX with IMAC has a similar or even poor performance than QMIX with unlimited communication. As the training process going, QMIX with IMAC has a better performance than QMIX with unlimited communication. The result shows that IMAC can facilitate different multi-agent methods which have different centralized training schemes.

\textbf{Performance under stronger limited bandwidth.} We evaluate agents' performance under different limited bandwidth constraints. Results show a similar conclusion as in previous tasks (Details can be seen in the supplementary materials).

\vspace{-10pt}

\begin{figure}[htbp]
  \includegraphics[width=\linewidth]{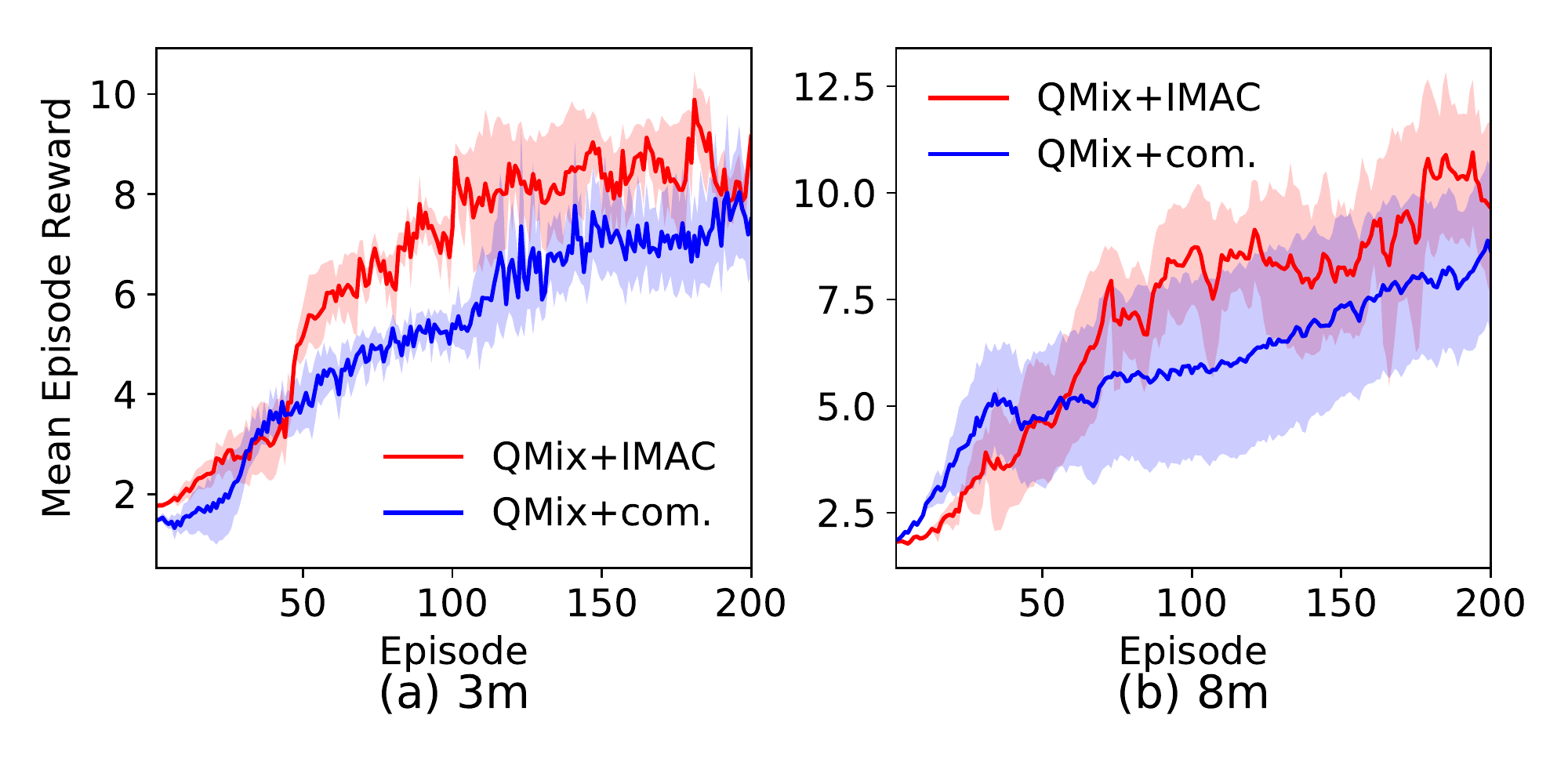}
  \centering
  \vspace{-25pt}
  \caption{Learning curves comparing IMAC to other methods for 3m and 8m in Starcraft II.}
  \label{fig:QMIX}
\end{figure}

\section{Conclusions}
In this paper, we have proposed an {\it informative multi-agent communication} method in the limited bandwidth environment, where agents utilize the information bottleneck principle to learn an informative protocol as well as scheduling. We have given a well-defined explanation of the limited bandwidth constraint from the perspective of communication theory. We prove that limited bandwidth constrains the entropy of the messages. We introduce a customized batch-norm layer, which controls the messages’entropy to simulate the limited bandwidth constraint. Inspired by the information bottleneck method, our proposed IMAC algorithm learns informative protocols and a weight-based scheduler, which convey low-entropy and useful messages. Empirical results and an accompanying ablation study show that IMAC significantly improves the agents’ performance under limited bandwidth constraint and leads to faster convergence.

\bibliography{example_paper}
\bibliographystyle{icml2020}

\end{document}